\documentclass[letterpaper, 10 pt, conference]{ieeeconf}  

\IEEEoverridecommandlockouts                              
                                                          
\usepackage{cite}
\usepackage{amsmath,amssymb,amsfonts,bm}
\usepackage[hidelinks,colorlinks]{hyperref}

\usepackage{algorithm}
\usepackage{algpseudocode}

\usepackage{graphicx}
\usepackage{textcomp}
\usepackage{xcolor}
\usepackage{subfigure}
\usepackage{mdframed}

\newcommand*{\QEDA}{\hfill\ensuremath{\triangle}}

\newcommand{\R}{\mathbb{R}}

\newcommand{\E}{\operatorname{\mathbb{E}}}
\newcommand{\BE}{\operatorname{\mathbb{E}}}
\newcommand{\cP}{\mathcal{P}}
\newcommand{\I}{\mathcal{I}}

\usepackage{mathtools}
\DeclarePairedDelimiter{\ceil}{\lceil}{\rceil}

\newcommand{\Span}{{\rm span}}

\newcommand*{\QE}{\hfill\ensuremath{\blacksquare}}%
\newenvironment{pfof}[1]{\vspace{1ex}\noindent{\itshape Proof of
		#1:}\hspace{0.5em}} {\hfill\QE\vspace{1ex}}

\newcommand{\alias}{C:/Users/yuzhe/Dropbox/bib/alias}
\newcommand{\New}{C:/Users/yuzhe/Dropbox/bib/New}
\newcommand{\Main}{C:/Users/yuzhe/Dropbox/bib/Main}
\newcommand{\FP}{C:/Users/yuzhe/Dropbox/bib/FP}

\DeclareMathOperator*{\argmax}{arg\,max\,\,}

\newtheorem{theorem}{Theorem}[section]
\newtheorem{lemma}[theorem]{Lemma}

\newtheorem{assumption}[theorem]{Assumption}

\makeatletter
\renewcommand*{\@opargbegintheorem}[3]{\trivlist
	\item[\hskip \labelsep{\itshape #1\ #2}] \textbf{(#3)}:\ \itshape }
\makeatother

\def\BibTeX{{\rm B\kern-.05em{\sc i\kern-.025em b}\kern-.08em
    T\kern-.1667em\lower.7ex\hbox{E}\kern-.125emX}}
\begin{document}

\title{\LARGE \bf Representation Learning for Context-Dependent Decision-Making
\thanks{Y. Qin and F. Pasqualetti are with the Department of Mechanical Engineering (\{yuzhenqin,fabiopas\}@engr.ucr.edu), and S. Oymak is with the Department of Electrical and Computer Engineering (oymak@ece.ucr.edu), University of California, Riverside, CA, USA. T. Menara is with the Department of Mechanical and Aerospace Engineering, University of California, San Diego, La Jolla, CA 92093, USA. S. Ching is with the Department of Electrical and Systems Engineering and Biomedical Engineering,	Washington University in St. Louis, MO, USA. This material was based upon work supported by awards ARO W911NF1910360 	and NSF NCS-FO-1926829. 
 }
}


\author{Yuzhen Qin, Tommaso Menara, Samet Oymak, ShiNung Ching, and Fabio Pasqualetti
}

\maketitle

\begin{abstract}
	Humans are capable of adjusting to changing environments flexibly and quickly. Empirical evidence has revealed that representation learning plays a crucial role in endowing humans with such a capability. Inspired by this observation, we study representation learning in the sequential decision-making scenario with contextual changes. We propose an online algorithm that is able to learn and transfer context-dependent representations and show that it significantly outperforms the existing ones that do not learn representations adaptively. As a case study, we apply our algorithm to the Wisconsin Card Sorting Task, a well-established test for the mental flexibility of humans in sequential decision-making. By comparing our algorithm with the standard Q-learning and Deep-Q learning algorithms, we demonstrate the benefits of adaptive representation learning. 
\end{abstract}

\section{Introduction}
Real-world decision-making is complicated since environments are often complex and rapidly changing. Yet, human beings have shown the remarkable ability to make good decisions in such environments. At the core of this ability is the flexibility to adapt their behaviors in different situations \cite{RA-SYS-NY:2021}. Such adaption is usually fast since humans learn to abstract experiences into compact representations that support the efficient construction of new strategies \cite{FNT-FMJ:2020}.



Lacking the ability to adapt to new environments and abstract compressed information from experiences, existing learning techniques often struggle in complex scenarios that undergo contextual changes. To elaborate on this point, let us consider a running example -- the Wisconsin Card Sorting Task (WCST). The WCST is one of the most frequently used neuropsychological tests to assess people's ability to abstract information and shift between contexts \cite{BB-GS-SW-BKF:2005}. Illustrated in Fig.~\ref{conceptual}, participants are initially given four cards and are required to associate a sequence of stimulus cards with these four cards according to some sorting rules -- number, color, and shape. Participants have no prior knowledge of the current sorting rule, thus need to learn it by trial and error. They receive a feedback indicating whether their sort action is correct or incorrect. What makes the task more challenging is that the sorting rule changes every once in a while without informing the participants. Thus, the participants need to learn the changes and adjust their strategy. 
 
Healthy humans usually perform very well in the WCST. Some neuroimaging studies have found that different brain regions, such as the dorsolateral prefrontal cortex and the anterior cingulate cortex, play crucial roles in context shifting, error detection, and abstraction, all of which are needed by the WCST \cite{LCH-SK-MJC-FGR:2006}. By contrast, classical learning algorithms such as tabular-Q-learning and Deep-Q-learning struggle in the WCST, especially when the sorting rule changes rapidly. It can be seen from Fig.~\ref{conceptual} that standard reinforcement learning (RL) algorithms\footnote{For the Deep-Q-learning, we considered a three-layer structure with 12 nodes in the hidden layers (more details can be found in Section~\ref{simulation}). Deeper or wider networks were also tried, but similar performances were observed.} perform barely better than the strategy that takes random sorting actions at every round.   

Motivated by these observations, we aim to develop decision-making strategies that have more human-like performance. In this paper, we focus on demonstrating the benefits of the ability to abstract compact information (i.e., learn the representation) and adapt to changing contexts in the framework of a sequential decision-making model -- linear multi-armed bandits. As we will show later, the WCST can be readily modeled in this framework. 
\begin{figure}
	\centering
	\includegraphics[scale=0.8]{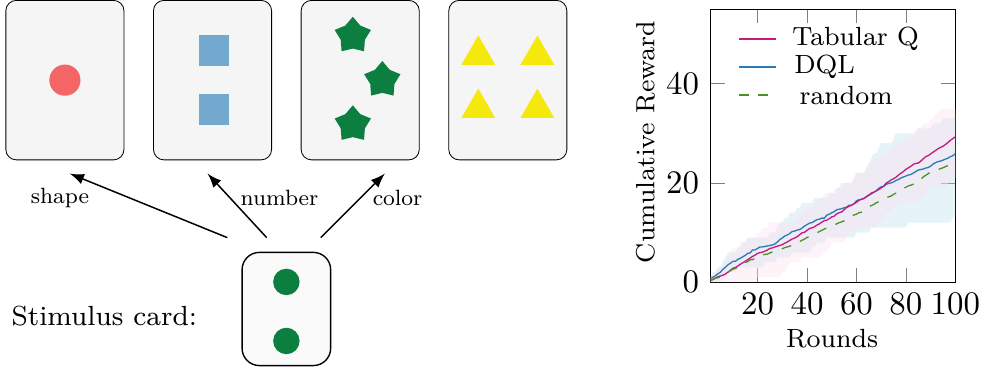}
	\caption{The Wisconsin Card Sorting Task. Left: Illustration of the task. Participants need to sort a sequence of stimulus cards into four categories according to unknown changing rules: number, color, and shape. Right: Performance of classical reinforcement learning algorithms in this task where the sorting rule changes after every 20 rounds. Here, we consider that participants receive reward 1 for a correct sorting action and 0 otherwise Each shaded area contains 20 realizations of the corresponding algorithm.}
	\label{conceptual}
\end{figure}

\textbf{Related Work.} As a classical model for decision-making, multi-armed bandits have attracted extensive interests. The Upper Confidence Bound (UCB) algorithm and its variants have proven their strength in tackling multi-armed bandit problems (e.g., see \cite{AP:2002,AYY-PD-SC:2011}). Various generalizations of the classical bandit problem have been studied, in which non-stationary reward functions \cite{RY-VC-CO:2019,WL-SV:2021}, restless arms \cite{GT-CK:2020}, satisficing reward objectives \cite{RP-SV-LNE:2016}, risk-averse decision-makers \cite{MM-CO:2021}, heavy-tailed reward distributions \cite{WL-SV:2020}, and multiple players \cite{HMK-DS:2021} are considered. Recently, increasing attention has been also paid to tackling bandit problems in a distributed fashion (e.g., see \cite{DK-NN-RJ:2014,LP-SV-LNE:2021,UM-NEL:2020,Zhu-Liu:2021}).  

Representation learning has been applied to a wide range of practical problems including natural language processing, computer vision, and reinforcement learning \cite{BY-CA-VP:2013}.
  Some recent studies have shown that representation learning improves data efficiency in the multi-task linear regression \cite{TN-JC-JMI:2021}.
Representation learning has been proven to be beneficial for multi-task bandit problems, e.g., see \cite{LS-AK-AA-HB:2019,JKS-WR-WS-NR:2019,LY-MA-TA:2020,Yang-Hu-Lee:2021}. Most of the aforementioned studies focus on batch learning where all the tasks are played simultaneously. Despite some attempts (e.g., see \cite{AMG-LA-BE:2013}), results on sequential bandits are sparse, although one often needs to execute tasks sequentially in real life. 

\textbf{Paper Contribution.} In this paper, we consider a decision-making scenario with changing contexts. A multi-task decision-making model with tasks sequentially drawn from distinct sets is used to describe a dynamic environment. Our main contribution is an algorithm that is able to abstract low-dimensional representations and adapt to contextual changes. We further derive some analytical results, showing the benefits of adaptive representation learning in complex and dynamic environments. To demonstrate our theoretical findings, we apply our algorithm to the WCST and show that it significantly outperforms classical RL algorithms. 

\textbf{Notation.} Let $\R$, $\R^+$, and $\mathbb Z^+$ be the sets of real numbers, positive reals, and positive integers, respectively. Given a matrix $A\in\R^{m\times n}$, $\Span(A)$ denotes its column space, $A_\perp$ denote the matrix with orthonormal columns that form the perpendicular complement of $\Span(A)$,  $\|A\|_F$ denotes its Frobenius norm, and $[A]_i$ denotes its $i$th column. For any $x\in R^+$, $\ceil{x}$ denotes the smallest integer larger than $x$. 
Given two functions $f,g:\R^+\to \R^+$, we write $f(x)=O(g(x))$ if there is $M_o>0$ and $x_0>0$ such that $f(x)\le M_og(x)$ for all $x\ge x_0$, and $f(x)=\tilde{O}(g(x))$ if $f(x)=O(g(x)\log^k (x))$. Also, we denote $f(x)=\Omega(g(x))$ if there is $M_\Omega>0$ and $x_0>0$ such that $f(x)\ge M_\Omega g(x)$ for all $x\ge x_0$, and $f(x)=\Theta(g(x))$ if $f(x)=O(g(x))$ and $f(x)=\Omega(g(x))$. 

\section{Problem Setup}\label{ProForm}
Motivated by real-world tasks like the WCST, we consider the following sequential decision-making model:
\begin{align}
	y_t = x_t^\top \theta_{\sigma(t)}+\eta_t, \label{noise}
\end{align}
where $x_t \in \mathcal A \subseteq \mathbb{R}^d$ is the action taken from the action set $\mathcal A$ at round $t$, and $y_t \in \mathbb{R}$ is the reward received by the agent (i.e., decision maker). The reward depends on the action in a linear way determined by the \textit{unknown} coefficient $\theta_{\sigma(t)}$, and is also affected by the $1$-sub-Gaussian noise $\eta_t$ that models the uncertainty. To make good decisions, the agent needs to learn $\theta_{\sigma(t)}$ under the influence of uncertainty.
This decision-making model is also known as linear bandits \cite{Dani-Hayes-Kakade:2008}. Note that the coefficient $\theta_{\sigma(t)}$ is time-varying, and $\sigma(t)$ is the switching signal. For simplicity, we assume that each task is played for $N$ rounds, i.e., $\sigma(t)$ changes its value after every $N$ rounds. Further, we assume that the agent plays $S$ tasks in total, and denote $\mathcal S=\{\theta_1,\theta_2,\dots,\theta_S\}$ as the task sequence. 

To model the context changes that underlie real-world tasks like the WCST, we assume that $\theta_{\sigma(t)}$ takes values from different sets. Specifically, we assume there are $m$ sets $\mathcal S_1,\mathcal S_2,\dots,\mathcal S_m$ from which $\theta_{\sigma(t)}$ takes values in sequence. In each $\mathcal S_k$, there are $n_k$ ($n_k$ can be infinite) tasks $\theta^k_1,\dots,\dots,\theta^k_{n_k}$, and we assume that they share a common linear feature extractor.  Different sets have different feature extractors. Specifically, there is $B_k\in\R^{d \times r_i}$ with orthonormal vectors such that for any $\theta^k_i$ there exists $\alpha_i^k\in \R^{r_i}$ so that $\theta^k_i=B_k \alpha^k_i$  (see Fig.~\ref{Pro:setup}).
 For simplicity, we assume that all the extractors have the same dimension $r$, i.e., $r_i=r$ for all $i$.  
 Here, each of these mutually different matrices $B_1,\dots,B_m$ are also referred to as a linear \textit{representation} \cite{HJ-CX-JC-LL-WL:2021} for the tasks in the respective set. 
 
 As for real-world problems like the WCST, $B_k$ describes the low-dimensional information that participants can abstract. For different contexts, participants usually need to abstract distinct low-dimensional features. 
Similar to the WCST in which participants do not know when the sorting rule changes, we further assume that the agent is not informed when $\theta_{\sigma(t)}$ starts to take values from a different task set. Denote $\tau_k, k=1,\dots,m$, as the \textit{unknown} number of sequential tasks that $\theta_{\sigma(t)}$ takes from $\mathcal S_k$.

\begin{figure}
	\centering
	\includegraphics[scale=0.58]{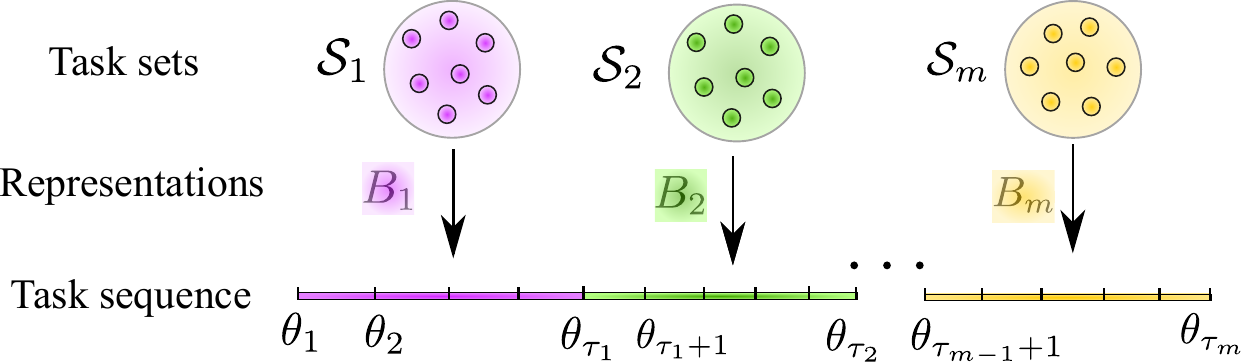}
	\caption{Sequential decision-making scenario with contextual changes. Tasks are taken from distinct sets in sequence. The tasks in each set share a low-dimensional representation. The length of each subsequence is unknown.}
	\label{Pro:setup}
	\vspace{-13pt}
\end{figure}

The agent's goal is to maximize the cumulative reward over the course of $SN$ rounds. 
To measure the performance, we introduce the regret
$
	R_{SN}= \sum_{t=1}^{SN} (x_{t}^*-x_{t})^\top\theta_{s(t)},
$
where $x^*_{t}$ is the optimal action that maximizes the reward at round $t$. Given $\theta$, denote $g(\theta)=\argmax_{x \in \mathcal A} x^\top \theta$, and then $x^*_t=g(\theta_{s(t)})$. The agent's objective is then equivalent to minimizing the regret $R_{SN}$.


We next make some standard assumptions on the action set $\mathcal A$ and the task coefficients following existing studies (e.g., see \cite{RP-TJN:2010,LY-WY-CX-ZY:2021}), which are considered to be satisfied throughout the remainder of this paper. 
\begin{assumption}\label{Assump:1}
	We assume that: (a) the action set $\mathcal A$ is a unit ball centered at the origin, i.e., $\mathcal A:=\{x\in \mathbb R^d: \|x\|\le 1\}$, and (b) there are positive constants $\phi_{\min}$ and $\phi_{\max}$ so that $\phi_{\min}\le \|\theta_s\| \le \phi_{\max}$ for all  $s \in \{1,2,\dots,S\}$.
\end{assumption}


Inspired by humans' strategy, we seek to equip the agent with the ability to learn and exploit representations and to quickly adjust to contextual changes so that it can perform well even in complex environments with context changes.  

\section{Adaptive Representation Learning}
In this section, we present our main results. We first analytically demonstrate why representation learning is beneficial especially for complex tasks that have high dimensions. Second, we propose a strategy to explore and transfer the representation under the setting of sequential tasks. Finally, we present our main algorithm that has the ability to adjust to contextual changes.

\subsection{Benefits of representation learning}
To demonstrate the benefits of representation learning, we first restrict our attention to a single-task model
\begin{align}\label{single}
	y_t=x_t^\top \theta +\eta_t,
\end{align}
where the task $\theta$ is played for $N$ times. For this classical model, existing studies have established the lower bound for its regret \cite{RP-TJN:2010,Dani-Hayes-Kakade:2008}, presented in the next lemma.

\begin{lemma}[\textbf{Classical Lower Bound}]\label{Lower:single}
	Let $\mathcal P$ be the set of all policies, and $\mathcal I$ be the set of	all the possible tasks. Then, for any $d\in \mathbb{Z}^+$ and $N>d^2$, the regret  $R_N$ for the task \eqref{single} satisfies $\inf_{\cP} \sup_{\I} \E R_N =\Omega(d\sqrt{N}).$ {\hfill \QEDA}
\end{lemma}

This lemma indicates that there is a constant $c>0$ such that the expected regret incurred by any policy is no less than $cd\sqrt{N}$ for any $d\in\mathbb Z^{+}$ and $N>d^2$. Next, we show how some additional information on $\theta$ affects this lower bound. 

\begin{lemma}[\textbf{Lower Bound with a Representation}]\label{Lower:single_rep}
	Suppose there is a known matrix $B\in \R^{d\times r}$ with $r<d$ such that $\theta=B\alpha$ for some $\alpha \in \R^r$. Let $\mathcal P$ be the set of all policies, and $\mathcal I$ be the set of	all the possible tasks. Then, for any $d\in \mathbb{Z}^{+}$ and $N>d^2$, the regret  $R_N$ for the task \eqref{single} satisfies $\inf_{\cP} \sup_{\I} \E R_N =\Omega(r\sqrt{N}).$ {\hfill \QEDA}
\end{lemma}

\begin{proof}
	Let $z_t=B^\top x_t$, and then the model in \eqref{single} becomes $y_t=z_t^\top \alpha+\eta_t$. As a consequence, the problem becomes to deal with a task with dimension $r$ instead of $d$. Following similar steps as in \cite{RP-TJN:2010}, it can be shown that the minimax lower bound for the regret is $\inf_{\cP} \sup_{\I} \E R_N =\Omega(r\sqrt{N})$, which completes the proof.
\end{proof}

Comparing Lemma~\ref{Lower:single_rep} with Lemma~\ref{Lower:single}, one finds that the regret lower bound decreases dramatically if $r\ll d$. This is because, with the knowledge of the representation $B\in\R^{d\times r}$, one does not need to explore the entire $\R^d$ space to learn the task coefficient $\theta$ for decision-making. Instead, one only needs to learn $\alpha$ by exploring a much lower-dimensional subspace $\Span(B)$ and estimate $\theta$ by $\hat \theta =B\hat \alpha$. As a consequence, $\theta$ can be learned much more efficiently, which helps the agent make better decisions at earlier stages.

Yet, such a representation $B$ is typically unknown beforehand. The agent usually needs to estimate $B$ from its experiences before utilizing it. In the next subsection, we show how to explore and transfer the representation in the setting of sequential tasks. 


\subsection{Representation learning in sequential tasks} 
Representation learning in the setting of sequential tasks is challenging, particularly when the agent has no knowledge of the number of sequential tasks that share the same representation. There is a trade-off between the need to explore more tasks to construct a more accurate estimate of the underlying representation and the incentive to exploit the learned representation for more efficient learning and higher instant rewards.

To investigate how to balance the trade-off, we consider that the agent plays $\tau$ tasks in sequence, i.e., $\mathcal T=\{\theta_1,\theta_2,\dots,\theta_\tau\}$,  without knowing the number of tasks $\tau$. There is an unknown matrix $B\in\R^{d \times r}$ such that for any $i$ it holds that $\theta_i=B\alpha_i$ for some $\alpha_i\in\R^r$. In this setting we aim to find a representation learning policy for each task subsequence in Fig.~\ref{Pro:setup}, i.e., a within-context policy.

We propose an algorithm, the sequential representation learning algorithm (SeqRepL, see Algorithm~\ref{alg:SeqRepL}), that alternates between two sub-algorithms -- representation exploration (RE) and representation transfer (RT) algorithms. Let us first elaborate on these two sub-algorithms, respectively.

\begin{algorithm}[tb]
	\caption{Representation Exploration (RE)}
	\label{alg:RE}
	\begin{algorithmic}
		\footnotesize
		\State \textbf{Input:}  Horizon N, exploration length $N_1=\ceil{d\sqrt{N}}$
		\State \textbf{for} {$t=1: N_1$} \textbf{do} \hspace{5pt}
		take $x_t= a_i$, $i=({t-1\mod d})+1$, where
		
		\State \hspace{65pt}  $[a_1,\dots,a_d]$ is any orthonormal basis of $\R^d$;
		
		\State compute $\hat \theta=(X_{\rm re} X_{\rm re}^\top)^{-1}X_{\rm re}Y_{\rm re}$, where $X_{\rm re}=[x_1,\dots,x_{N_1}]$, $Y_{\rm re}=[y_1,\dots,y_{N_1}]^\top$;
		\State \textbf{for} {$t=N_1+1: N$} \textbf{do}  \hspace{5pt} take $x_t= \argmax_{x \in \mathcal A} x^\top \hat \theta$ 
	\end{algorithmic}	
\end{algorithm}

\textbf{\textit{1) RE algorithm.}}  RE, shown in Algorithm~\ref{alg:RE}, is an explore-then-commit (ETC) algorithm, which contains two phases: \textit{exploration} and \textit{commitment}, consisting of $N_1$ and $N-N_1$ rounds, respectively.  The central goal of RE is to construct an accurate $\hat \theta_i$ for each task so that the collection of $\hat \theta_i$'s can recover an accurate representation $\hat B$ (which will be shown soon). Meanwhile, we want to ensure that the algorithm does not incur too much regret. To strike the balance, we set the exploration length $N_1=\ceil{d\sqrt{N}}$. The exploration phase is accomplished on the entire $\R^d$ space, in which $d$ linearly independent actions are repeatedly taken in sequence. Then, $\theta$ is estimated by the least-square regression $\hat \theta=(X_{\rm re} X_{\rm re}^\top)^{-1}X_{\rm re}Y_{\rm re}$, where $X_{\rm re}=[x_1,\dots,x_{N_1}]$, $Y_{\rm re}=[y_1,\dots,y_{N_1}]^\top$. In the commitment phase, the greedy action $x_t= \argmax_{x \in \mathcal A} x^\top \hat \theta$ is taken.  Note that the choice of $N_1$ ensures that the upper bound of RE is $O(d\sqrt{N})$, which {matches the lower bound} in Lemma~\ref{Lower:single}. The proof follows similar steps as those for Theorem 3.1 in \cite{RP-TJN:2010}. 

%

%

\begin{algorithm}[t]\caption{Representation Transfer (RT)}\label{alg:RT}
	\begin{algorithmic}
		\footnotesize
		\State \textbf{Input:} Horizon $N$, $\hat B \in \R^{d\times r}$, exploration length $N_2=\ceil{r\sqrt{N}}$
		\State \textbf{for} {$t=1: N_2$} \textbf{do}
		\hspace{5pt} take $x_i= a'_i$, $i=({t-1\mod  r})+1$, where
		
		\State \hspace{65pt}  $[a'_1,\dots,a'_r]$ is any orthonormal basis of $\Span(\hat B)$;
		
		\State compute 
		$\hat \alpha=(\hat B^\top X_{\rm rt}X_{\rm rt}^\top \hat B)^{-1}\hat B^\top X_{\rm rt}Y_{\rm rt}$ and $\hat \theta= \hat B \hat \alpha$, where $X_{\rm rt}=[x_1,x_2,\dots,x_{N_2}]$ and $Y_{\rm rt}=[y_1,y_2,\dots,y_{N_2}]^\top$;
		\State \textbf{for} {$t=N_2+1: N$}  \textbf{do}
		\hspace{5pt} take $x_t= \argmax_{x \in \mathcal A} x^\top \hat \theta$ \hspace{5pt}	
	\end{algorithmic}
\end{algorithm}

\textbf{\textit{2) RT algorithm.}} RT, shown in Algorithm~\ref{alg:RT}, is also an ETC algorithm. Its key feature is the utilization of $\hat B$. Thanks to $\hat B$, the exploration phase of RT is just carried out in the $r$-dimensional subspace $\Span(\hat B)$. Consequently, much fewer exploration rounds are required ($N_2=\ceil{r\sqrt{N}}$ rather than $\ceil{d\sqrt{N}}$). In the exploration phase, $r$ linear independently actions in $\Span(\hat B)$ are repeatedly taken before the $N_2$ rounds are exhausted. Unlike RE wherein $\hat \theta$ is directly constructed, RT first estimates $\alpha$ by the least-square regression $\hat \alpha=(\hat B^\top X_{\rm rt}X_{\rm rt}^\top \hat B)^{-1}\hat B^\top X_{\rm rt}Y_{\rm rt}$ with $X_{\rm rt}=[x_1,x_2,\dots,x_{N_2}]$ and $Y_{\rm rt}=[y_1,y_2,\dots,y_{N_2}]^\top$, and then recovers $\theta$ by $\hat B \hat\alpha$. 


With a perfect estimate $\hat B=B$, RT can achieve a regret upper bounded by $O(r\sqrt{N})$. This can be proven straightforwardly since the original  model can be rewritten into a $r$-dimensional one $y_t=z_t^\top \alpha+\eta_t$ by letting $z_t=\hat B^\top x_t$.  Yet, constructing a perfect $\hat B$ is usually impossible given the noisy environment. The next theorem provides an upper bound for the regret of RT when there is some error between $\hat B$ and $B$.

\begin{theorem}[Upper Bound Given an Estimated Representation]\label{transfer:single}
	Assume that an estimate $\hat B$ of the true representation $B$ satisfies $\|\hat B^\top B_\perp\|_F \le \varepsilon$. If the agent plays the task described by Eq.~\eqref{single} for $N$ rounds using  Algorithm~\ref{alg:RT} with $\hat B$, then the regret satisfies $\E R_N = O( r \sqrt{N} + N \varepsilon^2 )$. {\hfill \QEDA}
\end{theorem}

The upper bound in Theorem~\ref{transfer:single} is less than the lower bound $\Omega(d\sqrt{N})$ in Lemma~\ref{Lower:single} if $\varepsilon< \sqrt{d}/N^{\frac{1}{4}}$. This implies that the knowledge of an imperfect estimate of the representation improves the performance as long as it is sufficiently accurate (i.e., small $\|\hat B^\top B_\perp\|_F$).

\begin{pfof}{Theorem~\ref{transfer:single}}
	Since $\theta=B\alpha$, then the model becomes $y_t= x_t^\top B\alpha+\eta_t$. From Algorithm~\ref{alg:RT}, it holds that
	$
		\hat \alpha = (\hat B^\top X_{rt}X_{rt}^\top \hat B)^{-1} \hat B^\top X_{rt}Y_{rt}.
	$
	Without loss of generality, we assume $N_2$ is a multiple of $r$. Then, it can be calculated that $X_{rt}X_{rt}^\top = \frac{N_2}{ r} AA^\top$ with $A= [ a'_1,\dots, a'_{ r}]$, then we have 
	$
		\hat \alpha = \big( {N_2}\hat B^\top AA^\top \hat B/ { r}\big)^{-1} \hat B^\top X_{rt}Y_{rt}.
	$
	As $Y_{rt}=X_{rt}^\top B \alpha +\eta$ with $\eta=[\eta_1,\dots,\eta_{N_2}]^\top$, we have 
	$
		\hat \alpha  =\big( {N_2}\hat B^\top AA^\top \hat B /{ r}\big)^{-1} {N_2}\hat B^\top AA^\top  B \alpha/{ r} +  \big( \frac{N_2}{ r}\hat B^\top AA^\top \hat B \big)^{-1}  \hat B^\top X_{rt} \eta.
	$
	As $\hat \theta=\hat B \hat \alpha$ and $\theta= B \alpha$, it follows that
	\begin{align*}
		\hat B \hat \alpha-B \alpha=&\underbrace{\hat B \big( {N_2}\hat B^\top AA^\top \hat B/{r} \big)^{-1} {N_2}\hat B^\top AA^\top B\alpha/r -B\alpha }_{s_1}\\
		&+\underbrace{\hat B \big( {N_2}\hat B^\top AA^\top \hat B /r\big)^{-1}  \hat B^\top X \eta}_{s_2}.
	\end{align*}
	Then, it holds that $\BE \left[ \|\hat \theta(c) -\theta\|^2 \right] \le \BE \|s_1\|^2 + \BE \|s_2\|^2$ since $\eta_t$ is an independent random variable with zero mean. It can be derived (more details can be found in the extended version of this paper \cite{QY=MT-OS-CS-PF:2022(arxiv)}) that $\BE \|s_1\|^2 \le 2 c \phi_{\max}^2 \varepsilon^2$ for some constant $c$ and $\BE \|s_2\|^2 \le { r }/{\sqrt{N}}$. Combining $\BE \|s_1\|^2$ and $\BE \|s_2\|^2$, we have  $\BE \left[ \|\hat \theta -\theta\|^2 \right] \le {  r }/{ \sqrt{N}} +2 c\phi_{\max}^2 \varepsilon^2$.
	
	From \cite{RP-TJN:2010}, it follows that
	\begin{align}\label{intermediate:1}
		\BE[\max_{x\in \mathcal A}& x^\top \theta -\max_{x\in \mathcal A}x^\top \hat \theta] \nonumber\\
		&\le J\frac{\hat r }{ \phi_{\min}\sqrt{N}} +2\frac{1}{\phi_{\min}} J\phi_{\max}^2(1+\mu) \varepsilon^2.
	\end{align}
	For the commitment phase, there are $N-N_2$ steps. Thus, the overall regret satisfies $\BE R_N \le N_2 \phi_{\max} + (N-N_2) \left( \max_{x\in \mathcal A} x^\top \theta -\BE \max_{x\in \mathcal A}x^\top \hat \theta \right)$.  Substituting Eq.~\eqref{intermediate:1} into the right-hand side we obtain $\BE R_N\le O(\hat r \sqrt{N} +N \varepsilon^2)$, which completes the proof.
\end{pfof}

\textbf{\textit{3) SeqRepL algorithm.}} Let us now present the main algorithm in this subsection, which performs sequential representation learning (SeqRepL). It operates in a cyclic manner, alternating between RE and RT (see Algorithm~\ref{alg:SeqRepL}). In each cycle, there are two phases. In the RE phase of the $n$th cycle, $L$ tasks are played using RE. Then, the representation are estimated. Specifically, let $\hat P=\sum \hat \theta_i \hat \theta_i^\top$, where $\hat \theta_i$'s are the learned coefficients in all the previous $n$ RE phases. Then, $\hat B$ is constructed by performing \textit{singular value decomposition} (SVD) to $\hat P$ in the following way:
\begin{align*}
	{\rm SVD:}\hspace{2pt}\hat P = [U_1, U_2]\Sigma V \hspace{12pt} \longrightarrow \hspace{12pt} \hat B= U_1,
\end{align*}
where the columns of $U_1\in \R^{d\times r}$ are the singular vectors that are associated with the $r$-largest singular values of $\hat P$. In the RT phase, $nL$ tasks are played using RT with the estimated $\hat B$. Notice that $L$ more tasks are played using RT in each cycle than the previous one. This alternating scheme balances representation exploration and transfer well. 

\begin{algorithm}[t]\caption{Sequential Representation Learning (SeqRepL)}\label{alg:SeqRepL}
	\begin{algorithmic}
		\footnotesize
		\State \textbf{Input:} $\mathcal S_{\tau}=\{\theta_1,\dots,\theta_\tau\}$, $L=c_1 r$, $\hat P=0_{d\times d}$ \hspace{5pt} \textbf{Initialize}: $n=1$;
		\State \textbf{for each cycle $n$:} 
		\State \hspace{10pt} \textbf{RE phase:} play $L$ tasks in $\mathcal S_{\tau}$ using RE algorithm, $\hat P=\hat P+\hat\theta_i \hat \theta_i^\top$,
		\State \hspace{48pt}$\hat B \leftarrow$ top $r$	singular vector of $\hat P$;
		\State 	\hspace{10pt} \textbf{RT phase:} play $nL$ tasks in $\mathcal S_{\tau}$ using RT algorithm with latest $\hat B$;
		\State update $n=n+1$.
	\end{algorithmic}
\end{algorithm}

Next, we make an assumption and provide an upper bound for SeqRepL.

\begin{assumption}\label{Assump:diversity}
		For the task sequence $\mathcal T=\{\theta_1,\dots,\theta_{\tau}\}$, suppose that there exists $L =c_1 r$ for some constant $c_1>0$ such that any subsequence of length $L$ in $\mathcal T$ satisfies $\sigma_{r}(W_{s}W_{s}^\top)\ge \nu>0$ for any $s$, where $W_{s}=[\theta_{s+1}, \dots,\theta_{s+\ell}]$ and $\sigma_r(\cdot)$ denote the $r$th largest singular value of a matrix. {\hfill \QEDA}
\end{assumption}

This assumption states that the sequential tasks covers all the directions of the $r$-dimensional subspace  $\Span(B)$, which ensures that $B$ can be recovered in a sequential fashion.

\begin{theorem}[Upper Bound of SeqRepL]\label{regret:seq:single_B}
	Let the agent play a series of tasks $\{\theta_1,\theta_2,\dots,\theta_\tau\}$ using SeqRepL in Algorithm~\ref{alg:SeqRepL}, where $\tau>r^2$. Suppose that Assumption  \ref{Assump:diversity} is satisfied, then the regret, denoted by $R_{\tau N}$, satisfies
	$
		\E R_{\tau N}= \tilde O\left( d  \sqrt{\tau r N}+  \tau r \sqrt{N} \right). 
	$ {\QEDA}
\end{theorem}

Note that if one uses a standard algorithm, e.g., a UCB algorithm \cite{Dani-Hayes-Kakade:2008} or a PEGE algorithm \cite{RP-TJN:2010}, to play the sequence of tasks without learning the representation, the optimal regret would be  $\Theta(\tau d\sqrt{N})$. This bound is always larger than the two terms in our bound since $\tau>r^2$. This indicates that our algorithm outperforms the standard algorithms that do not learn the representations.




\begin{pfof}{Theorem~\ref{regret:seq:single_B}}
	After the RE phase of $n$th cycle in the SeqRepL algorithm, it can be derived (more details can be found in the extended version of this paper \cite{QY=MT-OS-CS-PF:2022(arxiv)})  that the estimate $\hat B$  and the true representation $B$ satisfy
	$
		\|\hat B^\top B_\perp\|_F= \tilde O \Big( \frac{d}{ \nu}{\sqrt{\frac{1}{ nL  d\sqrt{N}}}}  \Big).
	$
	The regret incurred in this phase of the $n$th cycle, denoted by $R_{\rm RE}(n)$, satisfies $R_{\rm RE}=O(Ld\sqrt{N})$
	Then, $n L$ tasks are played in sequence utilizing the RT algorithm with input $\hat B$.  It follows from Lemma~\ref{transfer:single} that the regret in the RT phase of the $n$th cycle, denoted as $R_{\rm RT} (n)$, satisfies
	$
		\BE R_{\rm RT} (n) \lesssim   n L r \sqrt{N} + n L N \frac{d^2 }{ \nu^2}{\frac{1}{ n L  d\sqrt{N}}}
		 =  \tilde O( n L  r \sqrt{N} +d \sqrt{N}).
	$
	Observe that there are at most $ \bar L=\lceil\sqrt{{2\tau}/{L}}\rceil$ cycles in the sequence of length $\tau$ since $L \bar L+L \bar L(\bar L+1)/2 \ge \tau$. 	Summing up the regret in Phases 1 and 2 in every cycle, we obtain
	$
		\BE R_{N\tau} \lesssim  \bar L L d\sqrt{N} +\sum_{m=1}^{\bar L} \left(  n L r \sqrt{N} + d \sqrt{N}\right)
		 \le \bar L L d\sqrt{N} + \tau r \sqrt{N}  +\bar L  d \sqrt{N}.
	$
	Since $L=c_1r$ for some constant $c_1$ and $\bar L=\lceil\sqrt{{2\tau}/{L}}\rceil$, then
	$
		\E R_{\tau N} = \sum_{n=1}^{\bar L} R_{\rm RE} (n) + R_{\rm RT} (n) = \tilde O \left( d  \sqrt{\tau r N}+  \tau r_i \sqrt{N} + d \sqrt{{\tau N}/{r}}  \right),
	$
	which completes the proof.
\end{pfof}

\subsection{Representation learning with contextual changes}
Finally, we are ready to address the problem that we set up in Section~\ref{ProForm}, i.e., representation learning in sequential tasks with changing contexts. 

\begin{algorithm}[t]
	\caption{Outlier Detection (OD)}
	\label{alg:OD}
	\begin{algorithmic}
		\footnotesize
		\State \textbf{Input:} $\hat B \in \R^{d\times r}$, $n_{\rm od}$, generate a random orthonormal matrix $Q \in\R^{(d- r)\times n_{\rm od}}$, and let $M=\hat B_\perp Q$.
		\State \textbf{for} {$t=1,\dots,n_{\rm od}$} \textbf{do} \hspace{5pt} $x_t= \delta [M]_t\in \mathcal A$, collect $y_t$
		\hspace{8pt} \textbf{end for}
		\State \textbf{if} {$Y_{n_{\rm od}}\notin \mathcal C_{n_{\rm od}}$} \textbf{then} \hspace{5pt} outlier indicator $\mathbb I_{\rm od}=1$ \hspace{8pt} \textbf{end if}
		\
	\end{algorithmic}
\end{algorithm}

\begin{algorithm}[t]
	\caption{Adaptive Representation Learning (AdaRepL)}
	\label{alg:main}
	\begin{algorithmic}
		\footnotesize
		\State \textbf{Input:} $k_c$ \hspace{8pt} \textbf{Initialize:}  $n_c=0$ (outlier counter), $\hat B=I_d$
		\State \textbf{for}  {$\theta_1,\theta_2,\dots,\theta_S$} \hspace{5pt} \textbf{do:}
		\State	\hspace{10 pt}invoke OD algorithm, return $\mathbb I_{\rm od}$
		\State \hspace{10 pt}\textbf{if} {$\mathbb I_{\rm od}=1$} \textbf{do} \hspace{5pt} invoke RE algorithm, $P=P+\hat \theta_i \theta_i^\top$, $n_c=n_c+1$ 
		\State \hspace{10 pt}\textbf{else} \hspace{5pt} $n_c=0$, invoke the cyclic SeqRepL
		\State \hspace{10 pt}\textbf{end if}
		\State \hspace{10 pt}\textbf{if} $n_c=k_c$ \textbf{do} \hspace{5pt} restart SeqRepL \hspace{8pt} \textbf{end if}			
	\end{algorithmic}
\end{algorithm}

In the WCST, humans are able to realize of sorting rule changes quickly. Inspired by that, we equip our algorithm with the ability to detect context switches, which enables it to adapt to new environments.

As shown in Algorithm~\ref{alg:OD}, the key idea is to take $n_{\rm od}$ probing actions for every new task. These actions are randomly generated in the perpendicular complement of $\Span(\hat B)$. Specifically, we generate a random orthonormal matrix $Q\in \R^{(d-r)\times n_{\rm od}}$. The probing actions are taken from the columns of the matrix $M=\delta \hat B_\perp Q$, where $\delta>0$ ensures that the actions are within the action set $\mathcal A$. If the current task $\theta$ satisfies $\theta=\hat B\alpha$ for some $\alpha$, it holds that $y_t=x_t^\top \theta  +\eta_t=\eta_t$ since $Q^\top \hat B_\perp ^\top \hat B \alpha=0$. Therefore, if the received rewards considerably deviate from the level of noise, the new task is an outlier to the current context  (i.e., a task that does not lie in the subspace $\Span(B)$) with high probability.

Let $Y_{n_{\rm od}}=[y_1,\dots,y_{n_{\rm rsd}}]^\top$ collect the rewards. Also, we build a confidence interval for $Y_{n_{\rm od}}$, which is 
$	\mathcal C_{n_{\rm od}} = \left\{ Y_{\rm od}\in \R^{n_{\rm od}}:  \left| \|Y\|_2 - \sqrt{n_{\rm od}}  \right|\le \xi_{\rm od}\right\}$,
where $\xi_{\rm od}$ is the detection threshold chosen by the agent. If the observed $Y_{\rm od}$ is beyond $\mathcal C_{n_{\rm od}}$, we decide that the new task is an outlier.

The main algorithm in this paper, which we call \textit{Adaptive Representation Learning algorithm} (AdaRepL), is provided in Algorithm~\ref{alg:main}, which invokes both SeqRepL and OD sub-algorithms. The former well balances representation exploration and transfer in the sequential setting, and the latter enables the algorithm to adapt to changing environments. To make our algorithm robust to occasional outliers, we set a threshold $k_c$ so that the algorithm considers that a context switch has occurred only when $k_c$ outliers have been detected consecutively.

It is worth mentioning that with the aid of the OD algorithm, the agent can detect context changes with high probability by properly selecting the detection threshold $\xi_{\rm od}$ and the length of probing actions $n_{\rm od}$. Within each context, the regret of AdaRepL has an upper bound presented in Theorem~\ref{regret:seq:single_B}. Although context change detection incurs some regret, the overall performance will still surpass the standard algorithms that are unable to learn representations adaptively. We will verify this point in the next section by revisiting the WCST.

\section{Experimental Study of WCST}\label{simulation}

First, we provide more details on the tabular-Q learning and Deep-Q learning algorithms in  Fig.~\ref{conceptual}. We assume that the agent receives reward 1 if it takes the classification action $x_t$ satisfies $x_t=\theta_t$, otherwise, it receives reward 0.

For the tabular-Q learning, the problem is to construct the $4^3\times 4$ Q table. This is because there are $4^3$ possible stimulus cards (4 colors, 4 numbers, 4 shapes) and each stimulus card can be taken as a state, and there are 4 sorting actions.
 
For the Deep-Q learning, we formalize each input state by a 3-dimension vector $\rm (shape,number,color)^\top\in \{1,2,3,4\}^3$. The result shown in  Fig.~\ref{conceptual} is based on a three-layer network with 3, 12, and 4 nodes in the input, hidden, and output layers, respectively. We also considered deeper or wider structures but obtained similar performances. 

It can be observed from Fig.~\ref{conceptual} that these two algorithms struggle in the WCST. The reason is that a large number of samples (certainly more than $4^4$ samples) are needed to construct the Q table or train the network weights. However, if the sorting rule changes much earlier than $4^4$ rounds, it is impossible to find the optimal policy. Also, being unaware of the sorting rule changes worsens the performance.



Next, we demonstrate how our proposed algorithm, which explore and exploit the representation in the WCST and detect sorting rule changes, has a much better performance.

\begin{figure}[t]
	\centering
	\includegraphics[scale=1]{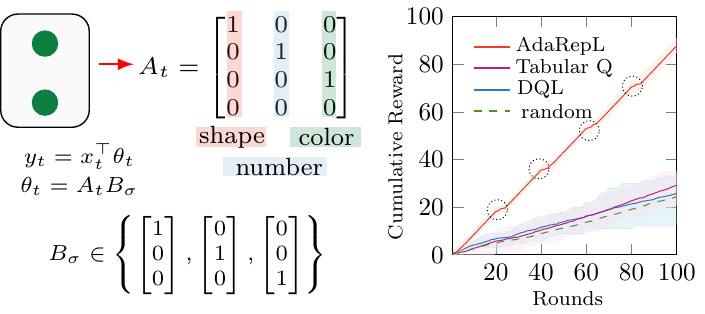}
	\caption{Left: key steps to model WCST into a sequential decision-making model with linear reward functions. Right: performance comparison between our algorithm and standard RL algorithms in WCST. Sorting rules change every 20 rounds. Dotted circles indicate that our algorithm is able to adapt to new contexts and learn new representations quickly.}
	\label{WCST:comparison}
\end{figure}

To do that, we model the WCST into a sequential decision-making model.  Specifically, we use a matrix $A_t \in \R^{4\times 3}$ to describe the stimulus card at round $t$. The first, second, and third columns of $A_t$ represent shape, number, and color, respectively, and they take values from the set $\{e_1,e_2,e_3,e_4\}$ with $e_i$ being the $i$th standard basis of $\R^4$. In each column, $e_i$ indicates that this card has the same shape/number/color as the $i$th card on table (see Fig.~\ref{WCST:comparison}). For example, the stimulus card (with two green circles) in Fig.~\ref{conceptual}  can be represented by the matrix $A=[e_1,e_2,e_3]$ (see Fig.~\ref{WCST:comparison}). Moreover, we use a standard unit vector $B_\sigma$, which takes values from $\{b_1,b_2,b_3\}$ with $b_i$ being the standard basis of $\R^3$, to respectively describe the 3 sorting rules -- shape, number, and color. In addition, the action $x_t$ also takes value from the set $\{e_1,e_2,e_3,e_4\}$. The action $x_t=e_i$ means to sort the stimulus card to the $i$th card on table. 

Consequently, the WCST can be described by the sequential decision-making model $y_t=x_t^\top \theta_t$ with $\theta_t=A_tB_\sigma$. Here the unit vector $B_\sigma$ can be taken as the current \textit{representation} since the correct sorting action can always be computed by $x_t^*=A_tB_\sigma$ no matter what card the agent sees. For instance, suppose the rule is number (i.e., $B_\sigma=b_2$), if the agent sees the stimulus card with two green circles, i.e., $A=[e_1,e_2,e_3]$, then correct sort is the second card on table since it can be computed that $x^*_t=A_tB_\sigma=[0,1,0,0]^\top$. 

The problem then reduces to learn the underlying representation $B_\sigma$, a task that is much easier than constructing the Q table or training the weights in a Deep-Q network. Remarkably, one does not even need to learn individual $\theta_t$ to construct $B_\sigma$. Instead, $B_\sigma$ can be recovered by 
$
B_\sigma=(\sum\nolimits_{t=1}^{k}A_t^\top x_t x_t^\top A_t)^{-1} \sum\nolimits_{t=1}^{k}A_t^\top x_t y_t 
$
immediately after $\sum_{t=1}^{k}A_t^\top x_t x_t^\top A_t$ becomes invertible. This indicates that our idea in this paper can apply to more general situations.

It can be observed in Fig.~\ref{WCST:comparison} that our algorithm significantly outperforms the other two, which demonstrates the power of being able to abstract compact representations and adapt to new environments.


\section{Concluding Remarks}\label{conclusion}
In this paper, we have studied representation learning for decision-making in environments with contextual changes. To describe such context-changing environments, we employ a decision-making model in which tasks are drawn from distinct sets sequentially. Inspired by strategies taken by humans, we propose an online algorithm that is able to learn and transfer representations under the sequential setting and has the ability to adapt to changing contexts. Some analytical results have been obtained, showing that our algorithm outperforms existing ones that are not able to learn representations. We also apply our algorithm to a real-world task (WCST) and verify the benefits of the ability to learn representations flexibly and adaptively. We are interested in studying representation learning in more general RL frameworks such as Markovian or non-Markovian processes.

\bibliographystyle{IEEEtran}
\bibliography{\alias,\FP,\Main,\New}


\end{document}